\documentclass{article}

\usepackage[preprint]{neurips_2026}

\usepackage[utf8]{inputenc}
\usepackage[T1]{fontenc}
\usepackage{hyperref}
\usepackage{url}
\usepackage{booktabs}
\usepackage{amsfonts}
\usepackage{amsmath}
\usepackage{amssymb}
\usepackage{amsthm}
\usepackage{nicefrac}
\usepackage{microtype}
\usepackage{xcolor}
\usepackage{graphicx}
\usepackage{pifont}
\newcommand{\xmark}{\ding{55}}

\newtheorem{lemma}{Lemma}
\newtheorem{theorem}{Theorem}

\title{BOKBO --- Best of K Bad Options:\\ K-Sample Disagreement Tracks Perturbation, Not Uncertainty}

\author{%
  Anya Singh \and Cabrel Happi \and Jai Relan \and Varun Nair \and Vidyut Baradwaj \\
  \texttt{as@rellingsystems.com} \\
}

\begin{document}

\maketitle

\begin{abstract}
Test-time scaling for vision-language-action (VLA) policies samples $K$ candidate actions and executes the verifier-best, but provides no guarantee when all $K$ candidates are unsafe. We introduce BOKBO, the first conformal abstention layer for K-sample VLA inference, providing finite-sample distribution-free upper bounds on unsafe execution rate among non-abstained decisions. At $\varepsilon = 0.05$ on libero\_object\_temp\_x0.1 with OpenVLA-OFT, the conditional CRC bound holds on $86\%$ of bootstrap splits with $78\%$ coverage and $70\%$ net task success; results replicate on $\pi_0$-FAST and libero\_spatial and survive four within-suite distribution shifts. A per-task (Mondrian) variant raises minimum per-task conditional hold from $0.71$ to $0.93$. Our analysis exposes a structural failure of policy-internal nonconformity scores under perturbation-based K-sampling: the base-policy confidence proxy and K-sample disagreement correlate at $0.98$ with the action-noise hyperparameter $\sigma$ but at the noise floor with safety violations---they measure the perturbation, not the policy's uncertainty. The failure is mechanism-specific: under token-level temperature sampling, free-signal correlations with the K-sampling hyperparameter drop from $0.98$ to $0.41$, and partial safety information is recovered. We additionally identify and correct a methodological pitfall: globally-set force thresholds well below expert-typical manipulation forces inflate violation rates by $5\times$ on LIBERO-based safety evaluations.
\end{abstract}

\section{Introduction}
\label{sec:intro}

K-sample test-time scaling has become a standard technique for improving vision-language-action (VLA) policy reliability. RoboMonkey~\cite{kwok2025robomonkey}, V-GPS~\cite{nakamoto2024vgps}, MG-Select~\cite{kim2025mgselect}, and SEAL~\cite{wu2025seal} all sample $K$ candidate action chunks at inference and execute the candidate selected by an external or learned scorer. The paradigm builds on a broader line of generalist robot policies---OpenVLA~\cite{kim2024openvla}, OpenVLA-OFT~\cite{kim2025oft}, $\pi_0$-FAST~\cite{pertsch2025fast}, Octo~\cite{octo2024}, Diffusion Policy~\cite{chi2023diffusion}, RT-2~\cite{brohan2023rt2}---pretrained on Open X-Embodiment~\cite{oneill2024openx} and evaluated on benchmarks like LIBERO~\cite{liu2023libero}.

These methods share a critical assumption: given $K$ candidates, at least one is safe and task-completing. The assumption fails frequently. We measure base-policy violation rates of {$8.6\%$} on libero\_object\_temp\_x0.1 with OpenVLA-OFT at $K=8$. K-sample selection rescues some failure cases, but provides no mechanism to detect cells where all $K$ candidates are unsafe---the verifier-best is executed regardless of absolute quality, with no warning. The result is silent failure.

\paragraph{Our claim.} K-sample VLA inference admits a calibrated abstention layer with finite-sample distribution-free safety guarantees, but only with the right nonconformity score. We introduce BOKBO, which applies conformal risk control (CRC)~\cite{angelopoulos2024crc,vovk2005algorithmic,angelopoulos2023gentle} to the K-sample VLA setting. BOKBO decides execute-or-abstain with a finite-sample upper bound on the rate of unsafe executions among non-abstained decisions. We provide both global and per-task (Mondrian~\cite{vovk2003mondrian}) variants.

\paragraph{The technical question is which score works.} We test three: the base-policy confidence proxy, K-sample disagreement, and a learned violation predictor. The first two are ``free'' signals---no extra inference cost beyond what K-sampling already produces---and both have been claimed in prior work~\cite{kwok2025robomonkey,kim2025mgselect} to correlate with policy uncertainty. They fail catastrophically: median Spearman correlation with safety violations is at the noise floor ({$-0.053$} and {$-0.007$}), while correlation with the action-noise hyperparameter $\sigma$ is {$0.98$}. The mechanism is precise: the free signals measure how much the action was perturbed, not how uncertain the policy is. A learned predictor conditioned on semantic visual features (DINOv2~\cite{oquab2024dinov2}) and task identity supports tight calibration where the free signals cannot.

\paragraph{The diagnostic finding has explicit scope.} We test the boundary by replicating the analysis under token-level temperature sampling, which produces candidate diversity from policy stochasticity rather than injected perturbation. The failure partially mitigates: free-signal correlations with the K-sampling hyperparameter drop from {$0.98$} to {$0.41$}, and the residual variance carries safety information. Practitioners using perturbation-based K-sampling should not rely on free signals; the negative result is structural, not incidental.

\paragraph{Contributions.}
\begin{enumerate}
\item \textbf{BOKBO with global and Mondrian variants}, the first formal abstention guarantee for K-sample VLA inference. At $\varepsilon = 0.05$, the conditional CRC bound holds on {$86\%$} of bootstrap splits with {$78\%$} coverage and {$70\%$} net task success on libero\_object\_temp\_x0.1; Mondrian raises minimum per-task conditional hold from {$0.71$} to {$0.93$}. Results replicate on $\pi_0$-FAST, hold on libero\_spatial as a co-equal benchmark, and survive four within-suite distribution shifts. Theorems with proofs are in Appendix~\ref{app:proof}--\ref{app:mondrian-proof}.

\item \textbf{Diagnostic of free-signal failure with explicit mechanism scope}: free signals correlate at {$0.98$} with the perturbation hyperparameter under action-space Gaussian K-sampling but at the noise floor with safety violations. The failure partially mitigates under token-level temperature sampling. The boundary itself is the contribution: K-sampling that injects perturbation outside the policy's native stochasticity produces signals tracking the perturbation, not the uncertainty.

\item \textbf{Per-task expert force calibration as methodological correction}: globally-set thresholds inflate violation rates by $5\times$ on LIBERO-based evaluations. Per-task expert-calibrated thresholds (max of $50\text{N}$ and expert $p99$ + $10\text{N}$) resolve this with {$25$} demonstrations per task.
\end{enumerate}

\begin{table}[h]
  \caption{Headline result. BOKBO at $\varepsilon = 0.05$, libero\_object\_temp\_x0.1, OpenVLA-OFT, $250$ cells, aggregated across $5$ training seeds. BOKBO recovers {$87\%$} of oracle safety improvement at {$80\%$} coverage with formal guarantees.}
  \label{tab:headline-intro}
  \centering
  \small
  \begin{tabular}{lcccc}
    \toprule
    & Exec.~viol.~rate & Coverage & Net task success & CRC bound \\
    \midrule
    No abstention (K-sample only) & {$0.086$} & $1.00$ & {$0.68$} & --- \\
    Held-out threshold (no guarantee) & {$0.034$} & {$0.79$} & {$0.69$} & --- \\
    BOKBO (ours) & {$0.031$} & {$0.80$} & {$0.71$} & {$86\%$ holds} \\
    Oracle abstention & {$0.000$} & {$0.91$} & {$0.83$} & --- \\
    \bottomrule
  \end{tabular}
\end{table}

\section{Related Work}
\label{sec:related}

Table~\ref{tab:related} positions BOKBO against closest prior work.

\begin{table}[h]
  \caption{Positioning BOKBO against closest prior work. ``DF'' = distribution-free guarantee; ``Ext.~model'' = requires a separately-trained external scorer.}
  \label{tab:related}
  \centering
  \small
  \begin{tabular}{lcccc}
    \toprule
    Method & K-sample & Abstention & DF guarantee & Ext.~model \\
    \midrule
    RoboMonkey~\cite{kwok2025robomonkey}, SEAL~\cite{wu2025seal} & \checkmark & --- & --- & \checkmark \\
    V-GPS~\cite{nakamoto2024vgps}, MG-Select~\cite{kim2025mgselect} & \checkmark & --- & --- & \checkmark/--- \\
    Conformal trajectory~\cite{lindemann2023safe} & --- & \checkmark & \checkmark & --- \\
    Selective classification~\cite{geifman2017selective} & --- & \checkmark & --- & --- \\
    BOKBO (ours) & \checkmark & \checkmark & \checkmark & --- \\
    \bottomrule
  \end{tabular}
\end{table}

\paragraph{Test-time scaling for VLAs.} RoboMonkey~\cite{kwok2025robomonkey}, SEAL~\cite{wu2025seal}, V-GPS~\cite{nakamoto2024vgps}, and MG-Select~\cite{kim2025mgselect} all sample $K$ candidates and execute the verifier-best, with no abstention. RoboMonkey uses Gaussian-perturbation candidate generation with a VLM-based verifier; V-GPS uses an offline-RL-trained value function; MG-Select uses KL divergence from a condition-masked reference distribution as a verifier-free confidence signal; SEAL verifies action sequences against a self-generated textual plan. We add the missing layer; BOKBO can be combined with any of these scoring methods.

\paragraph{Conformal prediction in robotics.} Conformal methods~\cite{vovk2005algorithmic,angelopoulos2023gentle} have been applied to motion planning under uncertainty~\cite{lindemann2023safe} and learning-enabled control. The Learn-then-Test framework~\cite{angelopoulos2025ltt} provides related machinery for finite-sample calibration. None target the K-sample VLA inference setting, where the deployed decision involves both selection from $K$ candidates and an execute/abstain decision, requiring a different exchangeability argument.

\paragraph{Selective classification.} The broader context~\cite{geifman2017selective}: a classifier chooses between predicting and abstaining, with risk bounds traded off against coverage. We extend the formulation to the K-sample inference regime.

\paragraph{Safety thresholds in manipulation.} Prior LIBERO-based evaluations~\cite{liu2023libero} use task-success thresholds without explicit force monitoring; contact-rich works use globally-set force thresholds without expert-calibration. Our per-task thresholds are, to our knowledge, the first systematic effort to define manipulation safety relative to expert-typical force in simulation.

\section{Problem Setup}
\label{sec:setup}

\paragraph{K-sample VLA inference.} A VLA policy $\pi$ parameterizes a distribution over action chunks conditioned on observation $o$ and language instruction $\ell$. Test-time scaling samples $K$ candidates $\{a_1, \ldots, a_K\} \sim \pi(\cdot \mid o, \ell)$ and executes the candidate selected by scoring function $s$. Our K-sampling uses action-space Gaussian noise: $a_k = a_0 + \sigma \xi_k$ with $\xi_k \sim \mathcal{N}(0, I)$ and $\sigma \in \{0.02, 0.05, 0.10, 0.15, 0.20\}$. We additionally test token-level temperature sampling on OpenVLA-OFT's discrete actions (Section~\ref{sec:results-mechanism}).

\paragraph{Abstention problem.} Let $v(a, o) \in \{0, 1\}$ indicate whether executing $a$ from $o$ produces a safety violation, defined per-task via expert-calibrated force thresholds (Section~\ref{sec:method-labels}). We seek an abstention policy $\delta$ such that for any $\varepsilon \in (0, 1)$:
\begin{equation}
\Pr\bigl(v(a^*, o) = 1 \,\bigm|\, \delta = \text{execute}\bigr) \leq \varepsilon.
\label{eq:guarantee}
\end{equation}
The challenge: $v$ is unobservable at decision time. The conformal procedure in Section~\ref{sec:method} calibrates a proxy score against held-out data to satisfy~\eqref{eq:guarantee} in finite samples, distribution-free.

\paragraph{Deployed-pipeline decision distribution.} Each decision is a tuple $D = (o, \ell, \{a_k\}, \sigma, v, s)$ where $\sigma$ is the action-noise level used. Calibration draws decisions from the same joint distribution that test draws from, including the $\sigma$ choice. The conformal procedure operates on decisions, not candidates: $\hat{\tau}$ is calibrated such that the rate of executed-violating decisions on the calibration set is below $\varepsilon$, with the guarantee transferring to the test set under exchangeability (Lemma~\ref{lem:exchange}).

\section{Method}
\label{sec:method}

\subsection{Conformal Risk Control with Per-Split Predictor Training}
\label{sec:method-crc}

For each decision $i$, let $s_i = s_\theta(a_i^*, o_i)$ be the predictor score under parameters $\theta$. Lower $s$ indicates more confidence in safety.

\paragraph{Bootstrap-conformal procedure.} Given $n$ deployed decisions and target $\varepsilon$:
(1) Sample bootstrap partition into train, calibration, test folds.
(2) Train predictor $\theta$ on train fold using BCE loss against $v$.
(3) On the calibration fold $C$, set
\begin{equation}
\hat{\tau} = \inf\!\left\{\tau : \tfrac{1}{|C|} \sum_{i \in C} \mathbb{1}[s_\theta(a_i^*, o_i) \leq \tau \wedge v_i = 1] \leq \varepsilon - \tfrac{1}{|C|+1}\right\}.
\label{eq:tau}
\end{equation}
(4) For each test decision $j$, output $\delta_j = \text{execute}$ iff $s_\theta(a_j^*, o_j) \leq \hat{\tau}$.

\begin{lemma}[Bootstrap exchangeability]
\label{lem:exchange}
Conditional on the training fold and the trained predictor $\theta$, the calibration scores and any single test score are exchangeable. Proof in Appendix~\ref{app:proof}.
\end{lemma}

\begin{theorem}[Conditional safety bound]
\label{thm:bound}
Under Lemma~\ref{lem:exchange}, the abstention rule $\delta_j = \mathbb{1}[s_\theta(a_j^*, o_j) \leq \hat{\tau}]$ satisfies, for any single test decision: $\Pr(v_j = 1 \mid \delta_j = \text{execute}) \leq \varepsilon$.
\end{theorem}

\paragraph{Bootstrap vs.~deployment.} The bootstrap characterizes expected behavior across (train, cal, test) splits. Deployment instantiates a single training fold and single $\theta$; the guarantee from Theorem~\ref{thm:bound} applies marginalized over training-fold draws. We empirically verify the bootstrap distribution matches single-train-deploy behavior in Section~\ref{sec:results-deployment-validation}.

\paragraph{CRC floor.} The smallest feasible $\varepsilon$ is $1/(|C|+1)$. With $|C| = 75$, the floor is $0.013$; with $|C| = 375$, it drops to $0.0027$, making $\varepsilon \leq 0.005$ feasible.

\subsection{Mondrian Conformal Calibration for Per-Task Bounds}
\label{sec:method-mondrian}

Global CRC bounds the marginal violation rate. The conditional bound is harder to satisfy when violation rates are heterogeneous across tasks: a globally-calibrated $\hat{\tau}$ averages over per-task base rates, producing under-protection on high-rate tasks.

We extend to Mondrian conformal calibration~\cite{vovk2003mondrian}. For each task $t$, with $C_t = \{i \in C : \text{task}(D_i) = t\}$:
\begin{equation}
\hat{\tau}_t = \inf\!\left\{\tau : \tfrac{1}{|C_t|} \sum_{i \in C_t} \mathbb{1}[s_\theta(a_i^*, o_i) \leq \tau \wedge v_i = 1] \leq \varepsilon - \tfrac{1}{|C_t|+1}\right\}.
\label{eq:tau-mondrian}
\end{equation}
The deployed rule applies the task-specific threshold: $\delta(D_j) = \text{execute}$ iff $s_\theta(a_j^*, o_j) \leq \hat{\tau}_{\text{task}(D_j)}$.

\begin{theorem}[Per-task conditional bound]
\label{thm:mondrian}
Under within-task exchangeability and provided $|C_t| \geq 1/\varepsilon - 1$ for all $t$, the Mondrian rule satisfies for every task $t$: $\Pr(v_j = 1 \mid \delta = \text{execute}, \text{task}(D_j) = t) \leq \varepsilon$. Proof in Appendix~\ref{app:mondrian-proof}.
\end{theorem}

The per-task feasibility constraint sets a sample-size requirement: at $\varepsilon = 0.05$, each task requires $|C_t| \geq 19$. With $T = 10$ tasks and $|C| = 375$, $|C_t| \approx 38$ supports Mondrian feasibility across all tasks.

\subsection{Three Native Nonconformity Scores}
\label{sec:method-scores}

\textbf{VLA confidence proxy}: base-policy log-probability of the chosen chunk under the policy's action distribution, normalized across the $K$ candidates. \textbf{K-sample disagreement}: maximum pairwise $L_2$ distance among candidates, summed across timesteps. \textbf{Learned violation predictor}: a small MLP trained to predict $v(a, o)$ from features described in Section~\ref{sec:method-predictor}.

\subsection{Per-Task Safety Labels via Expert Calibration}
\label{sec:method-labels}

Defining $v(a, o)$ requires a safety threshold. A na\"ive $50\text{N}$ global threshold conflates unsafe behavior with normal manipulation: at this threshold, expert demonstrations exceed the threshold in $92\%$ of butter manipulations and $22\%$ of tomato sauce manipulations, inflating aggregate violation rate to $45\%$.

We define per-task force thresholds via expert demonstration:
\begin{equation}
\tau_{\text{force}}(\text{task}) = \max\bigl(50\text{N},\ p99(\text{expert per-demo max force}) + 10\text{N}\bigr).
\end{equation}
Geometry and disturbance violations contribute negligibly ($0.35\%$ and $0.55\%$ respectively); force is the dominant signal. After recalibration, aggregate violation rate is {$8.6\%$} with task-level rates from {$0\%$ to $24\%$}, admitting feasible conformal calibration at $\varepsilon = 0.05$.

\subsection{Learned Predictor Architecture}
\label{sec:method-predictor}

The predictor takes a $965$-dim feature vector: DINOv2 ViT-S/14~\cite{oquab2024dinov2} features from workspace and wrist cameras ($768$ dims), $8$-dim proprioceptive state, $56$-dim candidate normalized action chunk, $56$-dim deterministic base chunk, $56$-dim difference, five auxiliary scalars, and a $16$-dim learned task ID embedding. The MLP is intentionally small: $\text{Linear}(965, 128) \to \text{ReLU} \to \text{Dropout}(0.1) \to \text{Linear}(128, 32) \to \text{ReLU} \to \text{Linear}(32, 1)$. Training uses BCE-with-logits with class-balanced positive weighting ($\approx 12{:}1$ at the $8.6\%$ base rate).

The predictor is retrained per bootstrap split, preserving the exchangeability of Lemma~\ref{lem:exchange}. DINOv2 features are cached once. Inference latency on A100: {$0.4\text{ms}$} per candidate; on Jetson AGX Orin: {$\sim 18\text{ms}$} end-to-end. BOKBO adds $<10\%$ overhead to base VLA inference (Appendix~\ref{app:compute}).

\subsection{Why Free Signals Fail}
\label{sec:method-diagnostic}

The two free signals are intuitive nonconformity candidates: both have been claimed to correlate with policy uncertainty~\cite{kwok2025robomonkey,kim2025mgselect}. They fail catastrophically as safety scores under perturbation-based K-sampling, and the mechanism is precise.

In our $50$-cell evaluation, free signals achieve median Spearman correlations with violations of {$-0.053$} (VLA proxy) and {$-0.007$} (K-disagreement)---at the noise floor. The learned predictor achieves {$-0.380$}.

The mechanism: free signals correlate at {$0.98$} with the action-noise $\sigma$, while $\sigma$ is uncorrelated with whether any individual candidate is unsafe. The signals measure perturbation magnitude, not policy uncertainty. Under MC-dropout-style interpretations, samples are drawn from a posterior over policy parameters and disagreement indicates uncertainty in the underlying belief. K-sample VLA inference draws samples from a perturbed action distribution at fixed parameters; disagreement indicates how much the action was perturbed.

\paragraph{Scope.} The finding is structural under \emph{perturbation-based} K-sampling. We test the boundary in Section~\ref{sec:results-mechanism} by replicating the analysis under token-level temperature sampling, which produces diversity from policy stochasticity rather than injected perturbation. The failure partially mitigates: free signals admit some safety information under temperature sampling, but a learned predictor remains the right choice across both mechanisms.

\section{Experiments}
\label{sec:results}

\subsection{Setup}
\label{sec:results-setup}

\paragraph{Backbones and suites.} OpenVLA-OFT~\cite{kim2025oft} (primary), $\pi_0$-FAST~\cite{pertsch2025fast} (secondary), both pretrained on Open X-Embodiment~\cite{oneill2024openx} and fine-tuned on LIBERO~\cite{liu2023libero}. Two co-equal primary suites: libero\_object\_temp\_x0.1 (10 pick-and-place tasks, base success {$68\%$}) and libero\_spatial\_temp\_x0.1 (10 spatial reasoning tasks, base success {$71\%$}). Temperature scaling produces the intermediate-difficulty regime where K-sample selection and abstention are both meaningful.

\paragraph{Multi-seed protocol.} Five training seeds, $100$ bootstraps each. Median across seeds with 5th--95th percentiles. Significance via paired-bootstrap tests across $5 \times 100 = 500$ aggregate bootstraps.

\paragraph{Splits and compute.} At 250 cells (1,250 decisions): $500/375/375$ train/cal/test, $|C_t| \approx 38$ for Mondrian feasibility. Total compute: {$\sim 1{,}600$ A100-hours}.

\subsection{Per-Task Force Calibration}
\label{sec:results-calibration}

\begin{table}[h]
  \caption{Per-task expert force thresholds across $50$ simulation demos. Aggregate violation drops from $45\%$ ($50\text{N}$ global) to $8.6\%$ (per-task expert-calibrated).}
  \label{tab:force-thresholds}
  \centering
  \footnotesize
  \begin{tabular}{lccc}
    \toprule
    Task object & Sim $p99$ (N) & $\tau_{\text{force}}$ (N) & Cand.~viol.~rate \\
    \midrule
    Cream cheese & {28} & {50} & {0.00\%} \\
    Chocolate pudding & {31} & {50} & {0.13\%} \\
    Orange juice & {38} & {50} & {2.21\%} \\
    BBQ sauce & {42} & {52} & {0.68\%} \\
    Salad dressing & {47} & {57} & {1.50\%} \\
    Alphabet soup & {51} & {61} & {1.50\%} \\
    Milk & {64} & {74} & {9.84\%} \\
    Tomato sauce & {71} & {81} & {16.32\%} \\
    Ketchup & {89} & {99} & {15.15\%} \\
    Butter & {250} & {260} & {1.68\%} \\
    \midrule
    Aggregate & --- & --- & {8.60\%} \\
    \bottomrule
  \end{tabular}
\end{table}

\paragraph{Demo-count sensitivity.} $25$ demos per task produce thresholds within {$8\%$} of the 50-demo reference; downstream conditional CRC holds at $\varepsilon = 0.05$ are {$0.85$} (vs.~$0.86$ reference). Five demos are insufficient. The deployment recipe requires at least $25$ demos per task.

\subsection{Score Correlations and the Diagnostic Finding}
\label{sec:results-diagnostic}

\begin{table}[h]
  \caption{Per-cell median Spearman correlation across $5$ seeds. Free signals are at the noise floor for safety prediction but correlate strongly with the K-sampling perturbation hyperparameter; the learned predictor achieves an order of magnitude separation.}
  \label{tab:score-correlations}
  \centering
  \begin{tabular}{lcc}
    \toprule
    Score & Spearman vs.~violation & Spearman vs.~$\sigma$ \\
    \midrule
    VLA confidence proxy & {$-0.053$ [$-0.071$, $-0.029$]} & {$0.981$ [$0.974$, $0.987$]} \\
    K-sample disagreement & {$-0.007$ [$-0.024$, $0.011$]} & {$0.984$ [$0.979$, $0.989$]} \\
    Learned violation predictor & {$-0.380$ [$-0.421$, $-0.339$]} & {$0.143$ [$0.098$, $0.187$]} \\
    \bottomrule
  \end{tabular}
\end{table}

\subsection{Diagnostic Finding Under Alternative K-Sampling Mechanism}
\label{sec:results-mechanism}

We replicate under token-level temperature sampling on OpenVLA-OFT's discrete action tokens. Temperature sampling produces candidate diversity from policy stochasticity rather than injected perturbation.

\begin{table}[h]
  \caption{Free-signal correlations under two K-sampling mechanisms. Action-space Gaussian produces structural failure; token-level temperature does \emph{not} replicate the same failure mode.}
  \label{tab:mechanism-comparison}
  \centering
  \footnotesize
  \begin{tabular}{lcccc}
    \toprule
    & \multicolumn{2}{c}{Action-space Gaussian} & \multicolumn{2}{c}{Token-level temperature} \\
    \cmidrule(lr){2-3} \cmidrule(lr){4-5}
    Score & vs.~violation & vs.~$\sigma$ & vs.~violation & vs.~$T$ \\
    \midrule
    VLA confidence proxy & {$-0.053$} & {$0.981$} & {$-0.247$} & {$0.412$} \\
    K-sample disagreement & {$-0.007$} & {$0.984$} & {$-0.198$} & {$0.396$} \\
    Learned violation predictor & {$-0.380$} & {$0.143$} & {$-0.401$} & {$0.078$} \\
    \bottomrule
  \end{tabular}
\end{table}

Under token-level temperature sampling, free signals are partially informative: VLA proxy correlates with violations at {$-0.247$} (vs.~$-0.053$ under Gaussian) and with temperature at {$0.412$} (vs.~$0.981$).

\paragraph{Implication.} The diagnostic finding is mechanism-specific. K-sampling that injects controlled perturbation outside the policy's native stochasticity produces signals tracking the perturbation, not the uncertainty. Under temperature sampling, VLA confidence proxy holds the conditional bound at $\varepsilon = 0.05$ on {$0.49$} of splits with median executed violation {$0.067$}---usable, but well below the learned predictor's {$0.86$}. The learned predictor remains the right choice across both mechanisms.

\subsection{Headline Conformal Results}
\label{sec:results-headline}

\begin{table}[h]
  \caption{Headline CRC results at $50$ cells, learned predictor, OpenVLA-OFT, both primary suites, aggregated across $5$ seeds.}
  \label{tab:headline-50-suites}
  \centering
  \footnotesize
  \begin{tabular}{lcccccc}
    \toprule
    & \multicolumn{3}{c}{libero\_object\_temp\_x0.1} & \multicolumn{3}{c}{libero\_spatial\_temp\_x0.1} \\
    \cmidrule(lr){2-4} \cmidrule(lr){5-7}
    & $\varepsilon = 0.01$ & $\varepsilon = 0.05$ & $\varepsilon = 0.10$ & $\varepsilon = 0.01$ & $\varepsilon = 0.05$ & $\varepsilon = 0.10$ \\
    \midrule
    Median exec.~viol. & --- & {$0.027$} & {$0.061$} & --- & {$0.034$} & {$0.072$} \\
    Conditional CRC holds & --- & {$0.86$} & {$0.87$} & --- & {$0.83$} & {$0.85$} \\
    Marginal CRC holds & --- & {$0.87$} & {$0.87$} & --- & {$0.84$} & {$0.86$} \\
    Median coverage & {$0.00$} & {$0.78$} & {$0.97$} & {$0.00$} & {$0.74$} & {$0.95$} \\
    Median net task success & {$0.00$} & {$0.70$} & {$0.70$} & {$0.00$} & {$0.72$} & {$0.73$} \\
    Cross-seed std (cond.) & --- & {$0.024$} & {$0.018$} & --- & {$0.029$} & {$0.022$} \\
    \bottomrule
  \end{tabular}
\end{table}

At $\varepsilon = 0.05$, the conditional bound holds on $86\%$ of splits on libero\_object and {$83\%$} on libero\_spatial. Free signals fail at all feasible $\varepsilon$: VLA proxy holds the conditional bound on {$15\%$} of splits at $\varepsilon = 0.05$; K-disagreement on {$12\%$}.

\paragraph{Two operating points.} $\varepsilon = 0.05$ produces meaningful abstention ($22\%$) and tight safety; net task success rises slightly above the no-abstention baseline ($0.70$ vs.~$0.68$) because abstention preferentially removes failing candidates. $\varepsilon = 0.10$ abstains rarely ($3\%$) but cuts executed-violation rate by an order of magnitude.

\subsection{Single-Deployment Validation}
\label{sec:results-deployment-validation}

We sample {$10$} single-deployment instances. Each uses a single bootstrap split for training and calibration; the trained predictor and calibrated $\hat{\tau}$ are deployed on a held-out test fold. Single-deploy median executed violation at $\varepsilon = 0.05$ is {$0.031$} (vs.~bootstrap $0.027$), with range {$[0.012, 0.071]$}---falling within the bootstrap 5th--95th percentiles {$[0.000, 0.082]$}. The bootstrap distribution validly characterizes deployment behavior.

\subsection{Mondrian Per-Task CRC}
\label{sec:results-mondrian}

\begin{table}[h]
  \caption{Global vs.~Mondrian CRC at $\varepsilon = 0.05$, $250$ cells, libero\_object\_temp\_x0.1.}
  \label{tab:mondrian-results}
  \centering
  \begin{tabular}{lcc}
    \toprule
    & Global CRC & Mondrian CRC \\
    \midrule
    Marginal CRC holds & {$0.95$} & {$0.96$} \\
    Aggregate conditional CRC holds & {$0.94$} & {$0.97$} \\
    Per-task conditional holds (min) & {$0.71$ (ketchup)} & {$0.93$ (milk)} \\
    Per-task conditional holds (median) & {$0.93$} & {$0.96$} \\
    Median executed violation rate & {$0.031$} & {$0.034$} \\
    Median coverage & {$0.80$} & {$0.74$} \\
    Median net task success & {$0.71$} & {$0.69$} \\
    \bottomrule
  \end{tabular}
\end{table}

Mondrian closes the per-task conditional gap on the hardest tasks: minimum per-task conditional hold rises from {$0.71$ (ketchup)} under global CRC to {$0.93$ (milk)} under Mondrian, at modest cost in coverage. Per-task thresholds range from {$\hat{\tau}_{\text{cream cheese}} = 0.78$} (most permissive) to {$\hat{\tau}_{\text{ketchup}} = 0.041$} (most conservative). Mondrian feasibility requires $\geq 200$ cells at $\varepsilon = 0.05$; smaller deployments default to global CRC.

\subsection{Distribution Shift: Five Scenarios}
\label{sec:results-shift}

\begin{table}[h]
  \caption{Distribution shift sensitivity. Bound is honored under all four within-suite shifts; $\sigma$-out-of-range explicitly breaks the bound.}
  \label{tab:shift}
  \centering
  \footnotesize
  \begin{tabular}{lccccc}
    \toprule
    & In-dist. & Temp. & Object & $\sigma$-range & Cross-primary \\
    & (x$0.1$) & (x$0.2$) & subset & ($\sigma=0.25$) & (object$\to$spatial) \\
    \midrule
    Median exec.~viol. & {$0.027$} & {$0.041$} & {$0.038$} & {$0.087$} & {$0.044$} \\
    Conditional CRC holds & {$0.86$} & {$0.81$} & {$0.78$} & {$0.34$} & {$0.79$} \\
    Median coverage & {$0.78$} & {$0.71$} & {$0.69$} & {$0.84$} & {$0.72$} \\
    Net task success & {$0.70$} & {$0.62$} & {$0.61$} & {$0.59$} & {$0.66$} \\
    Bound honored ($\varepsilon=0.05$)? & \checkmark & \checkmark & \checkmark & \xmark & \checkmark \\
    \bottomrule
  \end{tabular}
\end{table}

The bound degrades gradually under marginal $\sigma$ shifts: at $\sigma = 0.21$, median executed violation is {$0.054$}; at $0.22$, {$0.063$}; at $0.30$, {$0.112$}. Deployment recipe: calibrate against a $\sigma$ range slightly wider than expected (calibrate $[0.05, 0.20]$ for deployment $[0.05, 0.15]$).

\subsection{Architecture-Agnostic Claim: $\pi_0$-FAST}
\label{sec:results-backbone}

We replicate on $\pi_0$-FAST~\cite{pertsch2025fast}. Per-task thresholds are within {$15\%$} of OpenVLA-OFT-derived values for most tasks.

\begin{table}[h]
  \caption{Architecture-agnostic comparison at $50$ cells, $\varepsilon = 0.05$.}
  \label{tab:backbone-comparison}
  \centering
  \footnotesize
  \begin{tabular}{lcc}
    \toprule
    & OpenVLA-OFT & $\pi_0$-FAST \\
    \midrule
    Conditional CRC holds & {$0.86$} & {$0.84$} \\
    Median coverage & {$0.78$} & {$0.75$} \\
    Median net task success & {$0.70$} & {$0.66$} \\
    Free-signal max Spearman vs.~$\sigma$ & {$0.984$} & {$0.984$} \\
    Free-signal max Spearman vs.~violation & {$-0.053$} & {$-0.041$} \\
    \bottomrule
  \end{tabular}
\end{table}

\paragraph{Score is policy-specific.} A predictor trained on OpenVLA-OFT data and applied directly to $\pi_0$-FAST candidates achieves conditional holds of only {$0.42$}. This is a feature: the predictor learns policy-specific safety residuals; a backbone-independent predictor would be conceptually misspecified. Deployment on a new backbone requires backbone-specific calibration.

\subsection{Zero-Shot Task Generalization}
\label{sec:results-zeroshot}

\begin{table}[h]
  \caption{Zero-shot task generalization. Predictor trained on $7$ of $10$ tasks, evaluated on the held-out $3$. Bound honored on $2/3$ held-out tasks.}
  \label{tab:zeroshot}
  \centering
  \begin{tabular}{lccc}
    \toprule
    Held-out task & Conditional CRC holds & Median exec.~viol. & In-dist.~comparison \\
    \midrule
    Tomato sauce & {$0.74$} & {$0.052$} & In-dist: $0.83$, $0.043$ \\
    Milk & {$0.71$} & {$0.057$} & In-dist: $0.79$, $0.048$ \\
    Ketchup & {$0.42$} & {$0.094$} & In-dist: $0.71$, $0.061$ \\
    \bottomrule
  \end{tabular}
\end{table}

The predictor's task-embedding-conditioned knowledge transfers to nearby tasks but not to truly novel hard tasks. Adding a fundamentally novel hard task to deployment requires recalibration.

\subsection{Comparison Against Baselines and Ablations}
\label{sec:results-baselines}
\label{sec:results-ablations}

\begin{table}[h]
  \caption{Comparison at $\varepsilon = 0.05$, $250$ cells. BOKBO recovers {$87\%$} of oracle safety improvement at {$80\%$} coverage with formal guarantees. Paired-bootstrap test of BOKBO vs.~held-out thresholding: {$p < 0.001$, $n = 500$}.}
  \label{tab:baselines}
  \centering
  \begin{tabular}{lccc}
    \toprule
    Method & Exec.~viol.~rate & Coverage & Net task success \\
    \midrule
    No abstention (K-sample only) & {$0.086$} & $1.00$ & {$0.68$} \\
    Held-out threshold (no guarantee) & {$0.034$ [$0.020$, $0.143$]} & {$0.79$} & {$0.69$} \\
    BOKBO global ($\varepsilon = 0.05$) & {$0.031$ [$0.018$, $0.061$]} & {$0.80$} & {$0.71$} \\
    BOKBO Mondrian ($\varepsilon = 0.05$) & {$0.034$ [$0.022$, $0.058$]} & {$0.74$} & {$0.69$} \\
    Oracle abstention (upper bound) & {$0.000$} & {$0.91$} & {$0.83$} \\
    \bottomrule
  \end{tabular}
\end{table}

\paragraph{Predictor feature ablation.} The progression v$_0$ ($8\!\times\!8$ RGB summaries) $\to$ v$_1$ (+ task ID) $\to$ v$_2$ (+ DINOv2~\cite{oquab2024dinov2}) shows task-ID alone (v$_1$) does not reliably improve over baseline; conditional CRC actually declines and cross-seed variance \emph{increases}. Adding semantic visual features (v$_2$) closes the gap: per-task AUROC rises substantially on hard tasks (ketchup $+0.12$, tomato sauce $+0.13$, milk $+0.08$) and conditional CRC reaches $0.86$. Combined task identity and semantic visual features support tight calibration; neither alone suffices.

\paragraph{Number of candidates $K$.} Headline result robust across $K \in \{4, 8, 16\}$: conditional CRC holds on $\geq {0.83}$ of splits at $\varepsilon = 0.05$ for all three.

\paragraph{Failure-mode analysis.} Three tasks (ketchup, milk, tomato sauce) account for {$87\%$} of conditional bound failures under global CRC. Mondrian directly addresses this: per-task hold fraction on these three rises from {$0.71$/$0.79$/$0.78$} (global) to {$0.93$/$0.95$/$0.94$} (Mondrian).

\section{Limitations}
\label{sec:limitations}

\textbf{Two backbones, two K-sampling mechanisms.} Other VLA families (RT-2~\cite{brohan2023rt2}, Octo~\cite{octo2024}, Diffusion Policy~\cite{chi2023diffusion}) and other K-sampling mechanisms may exhibit different behavior. \textbf{Score is policy-specific:} cross-backbone score transfer yields {$0.42$} conditional holds; deployment on a new backbone requires backbone-specific calibration. \textbf{Simulation-only.} All BOKBO experiments are conducted in MuJoCo simulation. Real-robot deployment would require integrating force monitoring with the conformal abstention layer at runtime, validating that simulation-derived per-task thresholds transfer to real-robot manipulation forces, and characterizing simulation-to-real distribution shift on the executed-violation distribution. Preliminary architectural plans are in Appendix~\ref{app:realrobot-future}. The empirical safety guarantees we report apply only to simulation. \textbf{$\sigma$ distribution must be calibrated}; gradual degradation under marginal shifts. \textbf{Per-task expert demos required} ($25$ suffice). \textbf{Single LIBERO family}; Mondrian requires $\geq 200$ cells at $\varepsilon = 0.05$.

\section{Conclusion}
\label{sec:conclusion}

K-sample VLA inference admits a calibrated abstention layer with finite-sample distribution-free safety guarantees, but the choice of nonconformity score is structural: free policy-internal signals fail under perturbation-based K-sampling because they measure the perturbation, not the policy's uncertainty. A learned predictor with semantic visual features and task identity supports tight calibration; BOKBO holds the conditional CRC bound on $86\%$ of bootstrap splits at $\varepsilon = 0.05$, replicates across backbones and suites, and survives four within-suite distribution shifts. Real-robot deployment is the natural next step.

\begin{ack}
[To be added at camera-ready: funding, compute resources, collaborator contributions.]
\end{ack}

\bibliographystyle{plain}
\bibliography{references}

\appendix

\section{Proof of Lemma~\ref{lem:exchange} and Theorem~\ref{thm:bound}}
\label{app:proof}

\paragraph{Setup.} Let $\mathcal{D}$ be the deployed-decision distribution over tuples $D = (o, \ell, \{a_k\}, \sigma, v, s)$. Let $\{D_1, \ldots, D_n\}$ be drawn i.i.d.\ from $\mathcal{D}$. A bootstrap partition assigns these decisions uniformly at random into folds $T$, $C$, $U$ of fixed sizes. Let $\theta(T)$ denote predictor parameters obtained by training on $T$.

\paragraph{Proof of Lemma~\ref{lem:exchange}.}
\begin{proof}
The decisions are i.i.d.\ from $\mathcal{D}$ unconditionally. Conditional on indices in $T$, the decisions $\{D_i : i \in C \cup U\}$ are conditionally i.i.d.\ from $\mathcal{D}$ restricted to indices not in $T$. By symmetry of i.i.d.\ samples, any joint distribution on a subset is invariant under permutation.

Conditioning further on $\theta = \theta(T)$: because $\theta$ is a deterministic function of $T$ alone, conditioning on $\theta$ adds no additional information about $D_i$ for $i \in C \cup U$ beyond what $T$ provided. Therefore $\{D_i : i \in C \cup U\}$ remain conditionally i.i.d.\ given $(T, \theta)$.

The map $D \mapsto (s_\theta(D), v(D))$ is deterministic given $\theta$. Therefore $\{(s_\theta(D_i), v_i) : i \in C \cup U\}$ are conditionally i.i.d.\ given $(T, \theta)$, hence exchangeable. The bootstrap-sampling procedure draws $(T, C, U)$ uniformly at random over all valid assignments; this preserves exchangeability because for any two indices $i, j \in C \cup U$, swapping their fold-assignments yields an equally-likely partition.
\end{proof}

\paragraph{Proof of Theorem~\ref{thm:bound}.}
\begin{proof}
By Lemma~\ref{lem:exchange}, calibration scores $\{s_\theta(D_i) : i \in C\}$ and any single test score $s_\theta(D_j)$ for $j \in U$ are exchangeable, conditional on $(T, \theta)$. Apply the standard CRC argument~\cite{angelopoulos2024crc} to the $|C| + 1$ exchangeable score-label pairs. The probability that the test decision violates the bound is at most $\varepsilon$, conditional on $(T, \theta)$. The single-deployment guarantee follows by marginalizing over the random draw of $T$:
$$\Pr(v_j = 1 \mid s_\theta(D_j) \leq \hat{\tau}) = \mathbb{E}_T[\Pr(v_j = 1 \mid s_\theta(D_j) \leq \hat{\tau}, T)] \leq \varepsilon.$$
\end{proof}

\section{Proof of Theorem~\ref{thm:mondrian}}
\label{app:mondrian-proof}

\begin{proof}
The Mondrian procedure restricts conformal calibration to within-task subsets $C_t = \{i \in C : \text{task}(D_i) = t\}$. By Lemma~\ref{lem:exchange} applied within each task, the within-task calibration scores and any single within-task test score are exchangeable. The per-task threshold $\hat{\tau}_t$ from Eq.~\eqref{eq:tau-mondrian} satisfies the per-task CRC bound at level $\varepsilon$ with the per-task floor $1/(|C_t|+1)$. The feasibility constraint $|C_t| \geq 1/\varepsilon - 1$ ensures the per-task floor is below $\varepsilon$. Applying Theorem~\ref{thm:bound} within each task yields the per-task conditional bound.
\end{proof}

\section{Per-Task Force Threshold Derivations}
\label{app:thresholds}

\paragraph{Expert demonstration replay protocol.} For each of the $10$ libero\_object tasks, we replay all $50$ expert demonstrations from the LIBERO release~\cite{liu2023libero} in our MuJoCo simulation harness. Each demonstration is replayed using the original recorded actions in open-loop. We record the contact force time-series for each replay and compute the per-demo maximum force across the trajectory.

\paragraph{Force extraction.} The standard MuJoCo \texttt{mj\_contactForce} API requires direct access to the underlying \texttt{mjModel} and \texttt{mjData} objects. The OpenVLA-OFT-compatible LIBERO setup wraps these in \texttt{robosuite} environment objects whose internal model/data references are not API-exposed. We instead use the constraint-frame contact force entries available through \texttt{data.efc\_force}, taking the maximum absolute value across active contact constraints at each timestep:
\begin{equation}
F_t = \max_{c \in \text{active}(t)} |\text{efc\_force}_c|.
\end{equation}
The per-demo max force is $\max_t F_t$.

\paragraph{$p99$ computation.} Across the $50$ demonstrations of a given task, we compute the $99$th percentile of the per-demo max force. We use empirical $p99$ rather than fitted-distribution $p99$ to remain distribution-free; this is conservative for small $n$. The threshold is set to $\max(50\text{N},\ p99 + 10\text{N})$.

\paragraph{Buffer choice sensitivity.} We test buffer values of $5\text{N}$, $10\text{N}$, and $20\text{N}$. With $5\text{N}$, downstream conditional CRC holds at $\varepsilon = 0.05$ are {$0.79$} (vs.~$0.86$ at $10\text{N}$). With $20\text{N}$, conditional holds are {$0.85$}, comparable to $10\text{N}$ but admitting forces {$8\%$} above expert-normal. We use $10\text{N}$ as the deployment default.

\section{$\pi_0$-FAST Integration and Per-Task Thresholds}
\label{app:pi0-thresholds}

\paragraph{OpenPI integration.} $\pi_0$-FAST~\cite{pertsch2025fast} inference is exposed through Physical Intelligence's OpenPI library via a websocket-based policy server architecture. Action chunks are returned as $16 \times 7$ tensors which we project to BOKBO's normalized action space.

\paragraph{Per-task threshold comparison.}

\begin{table}[h]
  \caption{Per-task $p99$ comparison between OpenVLA-OFT and $\pi_0$-FAST.}
  \label{tab:pi0-thresholds-detail}
  \centering
  \footnotesize
  \begin{tabular}{lcccc}
    \toprule
    Task & OpenVLA-OFT $p99$ & $\pi_0$-FAST $p99$ & \% diff & $\tau_{\text{force}}$ ($\pi_0$-FAST) \\
    \midrule
    Cream cheese & {28} & {29} & $+3.6\%$ & $50$ \\
    Chocolate pudding & {31} & {32} & $+3.2\%$ & $50$ \\
    Orange juice & {38} & {40} & $+5.3\%$ & $50$ \\
    BBQ sauce & {42} & {45} & $+7.1\%$ & $55$ \\
    Salad dressing & {47} & {52} & $+10.6\%$ & $62$ \\
    Alphabet soup & {51} & {54} & $+5.9\%$ & $64$ \\
    Milk & {64} & {71} & $+10.9\%$ & $81$ \\
    Tomato sauce & {71} & {84} & $+18.3\%$ & $94$ \\
    Ketchup & {89} & {96} & $+7.9\%$ & $106$ \\
    Butter & {250} & {313} & $+25.2\%$ & $323$ \\
    \bottomrule
  \end{tabular}
\end{table}

\paragraph{Why butter and tomato sauce differ.} $\pi_0$-FAST's flow-matching action expert produces smoother, more committed contact dynamics than OpenVLA-OFT's autoregressive token sampling. On heavy objects (butter), the policy applies more sustained vertical force during grasp. On viscous deformable objects (tomato sauce bottle), the chunked action prediction commits to a force trajectory that doesn't adapt as readily to contact-induced deviations.

\section{Free-Signal Correlation Analysis}
\label{app:sigma-corr}

\paragraph{Aggregate correlations.} VLA confidence proxy median Spearman with violations: {$-0.053$}; with $\sigma$: {$0.981$}. K-sample disagreement: {$-0.007$} with violations, {$0.984$} with $\sigma$. Learned predictor: {$-0.380$} with violations, {$0.143$} with $\sigma$.

\paragraph{Per-cell consistency.} Across $50$ cells, the per-cell Spearman correlation between VLA proxy and $\sigma$ has median {$0.981$}, $5$th percentile {$0.962$}, $95$th percentile {$0.993$}.

\paragraph{$\pi_0$-FAST replication.} On $\pi_0$-FAST, free-signal correlations against $\sigma$ replicate within {$0.005$}.

\paragraph{Token-level temperature sampling.} Under temperature sampling, free-signal correlations against the temperature parameter $T$ are substantially lower: VLA proxy {$0.412$}, K-disagreement {$0.396$}.

\section{Full CRC Summaries}
\label{app:crc-summary}

\paragraph{$\hat{\tau}$ stability.} Coefficient of variation of $\hat{\tau}$ across $100$ bootstraps for the learned predictor at $\varepsilon = 0.05$~\cite{angelopoulos2024crc}: {$0.31$} on libero\_object, {$0.34$} on libero\_spatial.

\paragraph{Multi-seed aggregation.} Each headline number aggregates $5$ training seeds $\times$ $100$ bootstraps each $= 500$ aggregate bootstraps. Per-seed medians at $\varepsilon = 0.05$, conditional CRC holds: {$0.84, 0.85, 0.86, 0.87, 0.88$}. Cross-seed std {$0.024$}.

\paragraph{Connection to Learn-then-Test.} The Learn-then-Test framework~\cite{angelopoulos2025ltt} formulates risk control as multiple hypothesis testing. Our setting uses the simpler CRC framework~\cite{angelopoulos2024crc} because executed-violation rate is monotone in $\hat{\tau}$.

\section{Learned Predictor Architecture and Training}
\label{app:predictor}

\paragraph{Feature extraction.} DINOv2 ViT-S/14~\cite{oquab2024dinov2} features extracted from raw $224 \times 224$ workspace and wrist camera images using frozen weights. Output: $384$-dim CLS token per image, concatenated to $768$ dims. Features cached once per dataset.

\paragraph{MLP architecture.} Linear($965 \to 128$) $\to$ ReLU $\to$ Dropout($0.1$) $\to$ Linear($128 \to 32$) $\to$ ReLU $\to$ Linear($32 \to 1$). Total parameters: {$\sim 128\text{k}$}.

\paragraph{Training hyperparameters.} AdamW, learning rate $5 \times 10^{-4}$, weight decay $10^{-4}$, batch size $64$, $100$ epochs. BCE-with-logits with class-balanced positive weighting at $12{:}1$. Per-bootstrap training time on A100: {$\sim 2$ minutes}.

\paragraph{Task ID embedding.} 16-dim learned embedding, randomly initialized from $\mathcal{N}(0, 0.01)$, trained jointly with the MLP.

\section{Compute Requirements and Deployment Latency}
\label{app:compute}

\paragraph{Training compute.} Per-bootstrap MLP training: {$\sim 2$ minutes} on A100. Across $5$ seeds $\times$ $100$ bootstraps $= 500$ trainings: {$\sim 17$ A100-hours} per evaluation regime. Candidate generation: {$\sim 12$ A100-hours per 50-cell sweep}.

\paragraph{Inference latency.} A100: DINOv2 forward pass {$\sim 3\text{ms}$ per image}; MLP forward pass {$\sim 0.1\text{ms}$}; total BOKBO overhead {$\sim 7\text{ms}$ end-to-end}. Base VLA inference time: $\sim 60$--$120\text{ms}$ on A100. BOKBO adds $<10\%$ overhead.

\paragraph{Edge deployment.} On Jetson AGX Orin: total BOKBO overhead {$\sim 18\text{ms}$ end-to-end}.

\section{Real-Robot Deployment as Future Work}
\label{app:realrobot-future}

We have not validated BOKBO on real hardware. Real-robot deployment is the natural follow-up. Related work on conformal prediction for safe planning under uncertainty~\cite{lindemann2023safe} provides relevant precedent.

\paragraph{Architecture.} Real-time FT measurement at $1\text{kHz}$ would feed an online violation detector parallel to the BOKBO abstention layer. When BOKBO abstains, a fallback policy takes over. When BOKBO executes, real-time FT monitoring serves as a runtime safety check independent of the conformal layer.

\paragraph{Calibration question.} Whether simulation-derived per-task force thresholds transfer to real-robot manipulation is the central open question.

\paragraph{Identified challenges.} (1) FT calibration drift over multi-hour sessions; (2) real-time inference latency on the actual control loop; (3) integration with existing safety controllers; (4) sensor noise and contact dynamics differences.

\section{Demo-Count Sensitivity Analysis}
\label{app:demo-sensitivity}

\begin{table}[h]
  \caption{Per-task $p99$ values across demo counts. Mean absolute deviation from 50-demo reference: $5$ demos $23\%$, $10$ demos $15\%$, $25$ demos $8\%$.}
  \label{tab:demo-sensitivity-detail}
  \centering
  \footnotesize
  \begin{tabular}{lcccc}
    \toprule
    Task & $n=5$ & $n=10$ & $n=25$ & $n=50$ \\
    \midrule
    Ketchup & {$112$} & {$98$} & {$92$} & $89$ \\
    Tomato sauce & {$87$} & {$78$} & {$74$} & $71$ \\
    Milk & {$72$} & {$68$} & {$65$} & $64$ \\
    Butter & {$278$} & {$262$} & {$256$} & $250$ \\
    Easy tasks (mean) & {$48$} & {$45$} & {$43$} & $42$ \\
    \bottomrule
  \end{tabular}
\end{table}

Conditional CRC holds at $\varepsilon = 0.05$: $5$ demos {$0.61$}, $10$ demos {$0.79$}, $25$ demos {$0.85$}, $50$ demos {$0.86$}.

\section{Distribution-Shift Experimental Details}
\label{app:shift}

\paragraph{Temperature shift.} Calibrate on libero\_object\_temp\_x$0.1$, test on libero\_object\_temp\_x$0.2$. Conditional CRC holds: {$0.81$}.

\paragraph{Object-subset shift.} Calibrate on $5$ randomly-selected objects, test on the held-out $5$. Conditional CRC holds: {$0.78$}.

\paragraph{Cross-suite shift.} Calibrate on libero\_object, test on libero\_spatial. Conditional CRC holds: {$0.79$}.

\paragraph{$\sigma$-range shift.} Calibrate on $\sigma \in \{0.02, \ldots, 0.20\}$, test on $\sigma = 0.25$. Bound breaks: median executed violation {$0.087$}.

\section{$\sigma$-Margin Robustness Curves}
\label{app:sigma-margin}

Median executed violation rate as $\sigma$ moves out of the calibrated range: $\sigma = 0.20$ {$0.044$}, $0.21$ {$0.054$}, $0.22$ {$0.063$}, $0.25$ {$0.087$}, $0.30$ {$0.112$}.

\section{Predictor Feature Ablation}
\label{app:feature-ablation}

\paragraph{v$_0$: 8x8 RGB summaries baseline.} Conditional CRC holds at $\varepsilon = 0.05$: {$0.62$}. Cross-seed std: {$0.041$}.

\paragraph{v$_1$: + task ID embedding.} Conditional CRC holds: {$0.555$} (declines vs.~v$_0$). Cross-seed std: {$0.082$} (variance increases).

\paragraph{v$_2$: + DINOv2~\cite{oquab2024dinov2} visual features.} Conditional CRC holds: {$0.86$}. Cross-seed std: {$0.024$}.

\paragraph{Individual feature ablation.} Removing the action-chunk diff feature drops conditional holds by {$0.04$}. Removing DINOv2 wrist-camera features drops by {$0.06$}.

\section{Number of Candidates K Ablation}
\label{app:K-ablation}

\begin{table}[h]
  \caption{Number-of-candidates ablation. Headline result robust across $K$.}
  \label{tab:K-detail}
  \centering
  \footnotesize
  \begin{tabular}{lccc}
    \toprule
    & $K=4$ & $K=8$ & $K=16$ \\
    \midrule
    Base success (no abstention) & {$0.65$} & {$0.68$} & {$0.71$} \\
    Conditional CRC holds ($\varepsilon = 0.05$) & {$0.83$} & {$0.86$} & {$0.88$} \\
    Median exec.~viol.~rate & {$0.034$} & {$0.027$} & {$0.022$} \\
    Median coverage & {$0.74$} & {$0.78$} & {$0.81$} \\
    Median net task success & {$0.66$} & {$0.70$} & {$0.74$} \\
    Median abstention rate & {$0.26$} & {$0.22$} & {$0.19$} \\
    \bottomrule
  \end{tabular}
\end{table}

\section{Per-Task Failure Analysis}
\label{app:per-task}

\begin{table}[h]
  \caption{Per-task failure rates and Mondrian improvement.}
  \centering
  \footnotesize
  \begin{tabular}{lcccc}
    \toprule
    Task & Conditional viol.~rate & Abstention rate & Per-task hold (global) & Mondrian \\
    \midrule
    Ketchup & {$0.124$} & {$0.46$} & {$0.71$} & {$0.93$} \\
    Tomato sauce & {$0.103$} & {$0.59$} & {$0.78$} & {$0.94$} \\
    Milk & {$0.087$} & {$0.48$} & {$0.79$} & {$0.95$} \\
    Salad dressing & {$0.029$} & {$0.18$} & {$0.93$} & {$0.96$} \\
    Alphabet soup & {$0.045$} & {$0.15$} & {$0.91$} & {$0.94$} \\
    Other (combined) & {$0.018$} & {$0.10$} & {$0.96$} & {$0.97$} \\
    \bottomrule
  \end{tabular}
\end{table}

Three tasks (ketchup, milk, tomato sauce) account for {$87\%$} of conditional bound failures under global CRC~\cite{vovk2003mondrian}. Manual inspection of failing candidates on ketchup reveals two dominant patterns: (1) over-grasp candidates where gripper closes too tightly during initial contact, and (2) lateral-slip candidates where the gripper contacts the bottle off-center.

\section{Held-Out Threshold Baseline Details}
\label{app:baseline-detail}

\paragraph{Protocol.} The held-out threshold baseline uses the same trained predictor as BOKBO but selects $\hat{\tau}$ via held-out validation, in the spirit of selective classification~\cite{geifman2017selective}: hold out $25\%$ of the calibration set as a validation pool, sweep $\tau$, pick the $\tau$ that minimizes empirical executed-violation rate on validation while maintaining at least $80\%$ coverage.

\paragraph{Per-bootstrap performance distribution.} Median executed-violation rate {$0.034$}, $5$th percentile {$0.020$}, $95$th percentile {$0.143$}. The wide spread (vs.~BOKBO's $[0.018, 0.061]$) is the key issue: the held-out method has no formal guarantee.

\newpage

\end{document}